\documentclass[11pt]{article}

\usepackage[margin=1in]{geometry}
\usepackage[T1]{fontenc}
\usepackage[utf8]{inputenc}

\usepackage{amsmath,amssymb,amsthm,mathtools}

\usepackage[numbers]{natbib}
\usepackage[hidelinks]{hyperref}

\newtheorem{theorem}{Theorem}

\newtheorem{proposition}{Proposition}
\newtheorem{corollary}{Corollary}
\newtheorem{definition}{Definition}

\DeclareMathOperator{\E}{\mathbb{E}}

\newcommand{\SOTM}{\mathcal{M}}

\renewcommand{\Pr}{\mathbb{P}}
\newcommand{\Fset}{\mathcal{F}}
\newcommand{\Env}{\mathcal{E}}
\newcommand{\Task}{T}
\newcommand{\Transcript}{\Gamma}
\newcommand{\Oracle}{\mathcal{O}}
\newcommand{\Dist}{\mathcal{D}}
\newcommand{\Kb}{\mathcal{K}}
\newcommand{\tok}{\mathrm{tok}}
\newcommand{\TOK}{\mathrm{TOK}}

\newcommand{\korch}{\kappa_T^{\mathrm{orch}}}
\newcommand{\kag}{\kappa_T^{\mathrm{ag}}}

\title{\textbf{Computing with Agentic Oracles}}
\author{Jie Wang\,\thanks{Richard A. Miner School of Computing and Information Sciences,
University of Massachusetts, Lowell, MA 01854, USA.}\\[0.5ex]
\small\texttt{Jie\_Wang@uml.edu}}
\date{}

\begin{document}

\maketitle

\begin{abstract}
This paper extends the stochastic-oracle model of AI-augmented computing to
include agentic oracles. Unlike a stationary stochastic oracle, which responds to the
same query according to a fixed response distribution across calls, an agentic oracle
can pursue a goal autonomously and may access an environment containing
task-relevant resources. These capabilities affect both response
distributions and token costs beyond what is visible at the query-response
interface. We develop a framework for analyzing token costs in Stochastic-Oracle
Turing Machines (SOTMs) that compute with agentic oracles.
Each call has an \emph{orchestration token cost}, visible to the caller at the
query-response interface, and an \emph{agentic token cost}, incurred by
internal operations not exposed to the caller. We show that an SOTM computing
with an agentic oracle that can retain intermediate state can have token-cost
advantages over SOTMs using stationary stochastic oracles when solving the same
task at the same quality level, both with and without environment access. We
also investigate goal-loss risk,
including how internal dispatch ordering can reduce exposure to irreversible
actions. We provide a goal-loss avoidance criterion, derive
progress--retry--goal-loss formulas,
establish goal-depth lower bounds on token complexity, characterize token
complexity when the probability of goal loss is zero, and show that goal-loss
risk can impose an upper bound on the achievable quality of a task involving
environment updates.
\end{abstract}

\section{Introduction}

The Stochastic-Oracle Turing Machine (SOTM) framework~\cite{Wang2025,Wang2026,Wang2026comp} models
AI-augmented computation as a probabilistic Turing machine (PTM) directing a
computation with access to a stochastic oracle. The PTM performs computation,
decides when and what to query, and processes oracle responses, while the oracle
supplies knowledge or capabilities through responses drawn from query-dependent
response distributions. The oracle may
return the response generated by an underlying model directly, or it may use an
internal query-processing mechanism to route the query to one or several
models, aggregate model outputs, or follow a predetermined pipeline such as
retrieve-then-generate. In this paradigm, the stochastic oracle behaves
passively: the same query has the same response distribution on every call, and
the oracle does not autonomously access an environment.

Agentic oracles extend this paradigm by adding autonomous goal-directed control
and environment access as two distinct capabilities. On a query, an
agentic oracle may decompose the
goal, generate internal subqueries, invoke models or tools, evaluate
intermediate results, revise its plan, and repeat this loop before returning a
final response. Environment access may be read-only or read--write: the oracle
may inspect, modify, or otherwise act on task-relevant resources. Such an environment may
include documents, directories, codebases, databases, and APIs that the oracle
can inspect or modify, together with Web resources and other global resources
available as sources of information.
Either capability is sufficient for treating an oracle as agentic in this
paper, and in deployed systems they often appear together.

The PTM in an SOTM is the \emph{directing PTM}: it controls the
computation, forms oracle queries, processes oracle responses, and decides
when to halt or whether to continue interacting with a user or oracle.
To the directing PTM, the agentic oracle is a black box. The PTM submits a
query and receives a final response on its query-response tape, but it does not
control the oracle's internal decomposition, tool use, model calls, state
updates, or evaluation steps. These internal operations affect both the response
distribution and the token cost of computation. In particular, two calls with
similar query and response lengths may have very different internal
token costs, because one may require a short lookup while another may trigger a
long autonomous search, repair, or verification loop. This makes token-cost
analysis for agentic computation different from token-cost analysis for
stationary stochastic oracles~\cite{Wang2026,Wang2026comp,Wang2026prac,Wang2026cert}.

The first issue is token cost analysis. A call to an agentic oracle incurs two
types of token cost: an \emph{orchestration token cost}, visible on the PTM's
query-response tape, and an \emph{agentic token cost}, incurred by internal
operations not exposed to the PTM. The total execution-time token cost is the
sum of these two components. When
agentic token cost dominates, reducing the number of queries
made by the PTM may matter more than reducing the length of individual queries. When
orchestration cost dominates, the problem resembles the stochastic-oracle
setting. This distinction is especially important for commercial agentic
systems, where the internal dispatch mechanism, member models, and tool calls
may not be exposed through the query-response interface, while token expenditure
may be visible only through billing or provider-side usage reports.

The second issue is how \emph{agentic SOTMs}---SOTMs with an agentic oracle---can
benefit from delegation: they allow the directing PTM to delegate goal-directed
work rather than to code every intermediate step of a computation. This can
reduce development effort and, in some settings, can also reduce
execution-time token cost. An
SOTM using a stationary stochastic oracle can keep
transcripts and summaries on its own tapes, but information needed by the
stationary oracle must repeatedly cross the query-response interface. An
agentic oracle may instead retain the relevant state internally or internalize
it into its own mechanism. This creates token-cost advantages even when the
agentic oracle has the same knowledge and response-generating capacity as a
stationary stochastic oracle.

The third issue is goal-loss risk. In computation with a stationary stochastic
oracle, an erroneous response usually costs only another query
\cite{Wang2026cert,Wang2026comp}. In an
environment, an agentic oracle may take, recommend, or trigger an action that
changes the environment state. Some actions keep the computation on track:
the action makes the computation one step closer to achieving the goal. Other actions cause goal loss:
the goal was achievable before the action, but after the action the new state
renders the goal unachievable unless the computation can retract to a previous
state and apply different actions. Thus agentic computation has a risk
dimension absent from computation with stationary stochastic oracles. When every
solution path repeatedly incurs goal-loss risk, more queries do not
guarantee achieving the desired quality.

We investigate these issues in this paper. We recall the SOTM framework, tasks,
token counts, token costs, and token complexity in Section~\ref{sec:prelim}.
We define agentic oracles, environment tasks, agentic SOTMs, orchestration and
agentic token costs, and basic token-cost analysis results in
Section~\ref{sec:model}. We study token-cost advantages of SOTMs that compute
with agentic oracles capable of retaining intermediate state, compared with
SOTMs using stationary stochastic oracles, both with and without environment
access, in Section~\ref{sec:advantages}. We investigate irreversible actions
through goal loss, avoidance criteria, progress--retry--goal-loss formulas,
goal-depth lower bounds on token complexity, the case of zero goal-loss
probability, recovery of the stationary-oracle SOTM, and a task-level converse
for unavoidable goal-loss risk in Section~\ref{sec:foreclose}. Finally, we
conclude and describe future directions and open problems in
Section~\ref{sec:future}.

\section{Preliminaries: The SOTM Framework}
\label{sec:prelim}

We extend the Stochastic-Oracle Turing Machine (SOTM) framework introduced by
Wang~\cite{Wang2025}, using its later consolidated form
\cite{Wang2026,Wang2026comp,Wang2026prac},
which allows each query to include part or all of the prior query-response
transcript and permits fixed background information shared across all
instances of a task. This section recalls the necessary definitions. A fixed
finite alphabet $\Sigma$ is used to encode all queries, responses, inputs,
outputs, background materials, and environment information, including actions
and observations.

\paragraph{Stationary stochastic oracles.}
A \emph{stationary stochastic oracle} is a mechanism that, on a query string
$q\in\Sigma^\ast$, draws a response from a distribution $\Dist_q$ over
$\Sigma^\ast$. The query string $q$ encodes a query instruction.
It may also encode auxiliary information supplied by the caller, such as an
input instance, fixed background materials, or records of earlier query-response
pairs with the oracle. Each auxiliary component may be included
in full, in part, or omitted.
For every call on the same query string $q$, the response is drawn from the
same distribution $\Dist_q$.

When there is no confusion, we may use \emph{stochastic oracle} or
\emph{stationary oracle} to mean \emph{stationary stochastic oracle}. A
stationary stochastic oracle is denoted by $\Oracle$. A stationary stochastic
oracle may still have a nontrivial internal mechanism.
For example, it may route a query through a predetermined pipeline, call several
models, aggregate their outputs, or apply a predetermined retrieve-then-generate
procedure. It is stationary because repeated calls on the same query draw
responses from the same distribution.

\paragraph{Stochastic-Oracle Turing Machines.}
A \emph{Stochastic-Oracle Turing Machine} is a pair
$\mathcal{M}=(M,\Oracle)$, where $M$ is a probabilistic Turing machine (PTM)
and $\Oracle$ is a stochastic oracle. The PTM $M$ has no access to the
construction of $\Oracle$ or to responses drawn from $\Dist_q$ except by
querying $\Oracle$.

The PTM $M$ has one distinguished read-only input tape containing an input
string $x \in \Sigma^*$. It also has zero or more read-only background-input
tapes containing fixed background materials $\Kb_1,\ldots,\Kb_m$, together
with work tapes, an output tape, a random-source tape, and a query-response
tape. We write $\Kb=(\Kb_1,\ldots,\Kb_m)$ for the collection of background
materials. When no background material is supplied, $m=0$ and $\Kb=\emptyset$.
The full input context is $\mathcal I=x$ if $\Kb=\emptyset$ and
$\mathcal I=(x;\Kb)$ otherwise; formulas written with $\mathcal I$ use this
convention.

The query-response tape mediates the interaction between $M$ and $\Oracle$:
when $M$ writes a query string $q \in \Sigma^*$ on this tape, $\Oracle$ draws
$r \sim \Dist_q$ and writes $r$ back on the tape.

We write $\mathcal M(\mathcal I)$ to denote the output produced by this computation.
The computation proceeds in turns. Before the $i$-th oracle call, the input
string $x$, the background materials $\Kb$, and the prior transcript
$\Gamma_{i-1}=(q_1,r_1,\dots,q_{i-1},r_{i-1})$, with $\Gamma_0$ empty, are available to
$M$. Using any part of this information, together with its internal state and
random source, $M$ either halts and writes an output $y \in \Sigma^*$ or writes
the next query $q_i \in \Sigma^*$ on the query-response tape. In the latter
case, the oracle exchange produces $r_i \sim \Dist_{q_i}$, and the turn ends with
$\Gamma_i=(q_1,r_1,\dots,q_i,r_i)$.

\paragraph{Tasks.}
A \emph{task} is a tuple $T=(X,Y,S,\Dist_X,\Kb)$, where
$X\subseteq\Sigma^*$ is the input space, $Y\subseteq\Sigma^*$ is the output
space, $S:X\times Y\to[0,1]$ is the score function, $\Dist_X$ is an input
distribution on $X$, and $\Kb$ is fixed background information. The
background information $\Kb$ may be a document collection, domain resources, a
system prompt, or other materials placed on the background-input tapes of an
SOTM. The instance $x$ is the varying input to the SOTM, while $\Kb$ stays fixed
for the task. If no background material is supplied, $\Kb=\emptyset$.

\paragraph{Token count.}
A \emph{tokenizer} is a map $\tau$ that converts a string in $\Sigma^*$ into a
finite sequence of units called \emph{tokens}. Depending on the tokenizer, a
token may correspond to a word, a subword, punctuation, whitespace, a byte
sequence, or another atomic unit; common examples include subword tokenizers
such as byte-pair encoding and SentencePiece~\cite{Sennrich2016,Kudo2018}. A
\emph{token vocabulary} $\mathcal V$ is the set of tokens used by the
tokenizer. The \emph{token count} of a string $s$ is $\tok(s)=|\tau(s)|$.

For an SOTM $\mathcal{M}=(M,\Oracle)$ on input string $x$, with optional fixed
background information $\Kb$, let $q_i$ and $r_i$ be the query and response on
the $i$-th oracle call, and let $N(\mathcal I)$ be the possibly random number
of calls. The query and response token counts of turn $i$ are
$\tok_{\mathcal{M},Q,i}(\mathcal I)=|\tau(q_i)|$ and
$\tok_{\mathcal{M},R,i}(\mathcal I)=|\tau(r_i)|$.
The query token count, response token count, and total token count on
$\mathcal I$ are
\[
  \tok_{\mathcal{M},Q}(\mathcal I)
  =
  \sum_{i=1}^{N(\mathcal I)} \tok_{\mathcal{M},Q,i}(\mathcal I),
  \qquad
  \tok_{\mathcal{M},R}(\mathcal I)
  =
  \sum_{i=1}^{N(\mathcal I)} \tok_{\mathcal{M},R,i}(\mathcal I),
\]
and
$\tok_{\mathcal{M}}(\mathcal I)
=\tok_{\mathcal{M},Q}(\mathcal I)+\tok_{\mathcal{M},R}(\mathcal I)$.
Let $\tok_{\mathcal{M},Q}$, $\tok_{\mathcal{M},R}$, and
$\tok_{\mathcal{M}}$ without an input argument denote the expectations of these
quantities.

Throughout the paper, unless otherwise stated, expectations involving an SOTM
are taken over the analytical input distribution $x\sim\Dist_X$, the PTM's
random source, and the responses drawn from the oracle.
We write $\Pr[\cdot]$ for probability and use lower-case symbols such as
$p_j$, $\pi_j$, and $p_{\mathcal I}$ for probability parameters.

\paragraph{Token cost.}
Throughout the paper, token costs are computed with fixed unit query token cost
$\alpha>0$ and fixed unit response token cost $\beta>0$. The \emph{token cost}
of turn $i$ is $\TOK_{\mathcal{M},i}(\mathcal I;\alpha,\beta)
=\alpha\,\tok_{\mathcal{M},Q,i}(\mathcal I)
+\beta\,\tok_{\mathcal{M},R,i}(\mathcal I)$.
The query token cost, response token cost, and total token cost on
$\mathcal I$ are
$\TOK_{\mathcal{M},Q}(\mathcal I;\alpha,\beta)
=\alpha\,\tok_{\mathcal{M},Q}(\mathcal I)$,
$\TOK_{\mathcal{M},R}(\mathcal I;\alpha,\beta)
=\beta\,\tok_{\mathcal{M},R}(\mathcal I)$, and
\[
  \TOK_{\mathcal{M}}(\mathcal I;\alpha,\beta)
  =
  \TOK_{\mathcal{M},Q}(\mathcal I;\alpha,\beta)
  +
  \TOK_{\mathcal{M},R}(\mathcal I;\alpha,\beta)
  =
  \sum_{i=1}^{N(\mathcal I)}
  \TOK_{\mathcal{M},i}(\mathcal I;\alpha,\beta).
\]
Let $\TOK_{\mathcal{M},Q}(\alpha,\beta)$,
$\TOK_{\mathcal{M},R}(\alpha,\beta)$, and
$\TOK_{\mathcal{M}}(\alpha,\beta)$ without an input argument denote the
expectations of these quantities.
To keep notation readable, we may suppress $\alpha$ and $\beta$ from token-cost
notation when the dependence on these parameters is clear.
In many commercial LLM pricing schemes, response tokens are weighted above
query tokens ($\beta > \alpha$).

\paragraph{Token complexity.}
An SOTM $\mathcal M$ \emph{achieves quality} $\theta\in(0,1]$ on $T$ if
$\E[S(x,\mathcal M(\mathcal I))]\geq\theta$, where the expectation follows the
expectation convention.
The \emph{token complexity}
$\kappa_T(\theta; \Oracle, \alpha, \beta)$ is the infimum of
$\TOK_{\mathcal{M}}(\alpha,\beta)$ over all SOTMs
$\mathcal{M} = (M, \Oracle)$ with the same $\Oracle$ achieving quality
$\theta$ on $T$.

\section{Agentic Oracles}
\label{sec:model}

Extending the stochastic-oracle paradigm, an agentic oracle has at least one of
two additional capabilities: autonomous goal-directed control or environment
access. These capabilities need not appear together. An oracle may be agentic
because it autonomously decomposes a goal, generates intermediate queries,
evaluates partial results, revises its plan, and decides when to stop, even
without environment access. It may also be agentic because it can
access an environment containing task-relevant resources, even when its
internal control is otherwise simple. In many deployed systems the two capabilities
are combined: the oracle pursues a goal carried by the query while inspecting
or modifying files, codebases, databases, APIs, and tools, and accessing Web
resources or other global resources for information.

To the caller, both stationary stochastic and agentic oracles return a response,
and either may internally use several models and combine their outputs. What
distinguishes an agentic oracle is therefore not the use of multiple internal
models, but one or both of two capabilities: autonomous goal-directed control
and environment access. When
environment access is present, modeling the oracle requires an environment
with states, observations, available actions, a score function, and a transition
rule fixing how actions change the state. The environment can turn oracle use
into an adaptive process in which the oracle may inspect the current state,
take an action, verify the outcome, update the environment or its internal
state, and repeat.

This distinction is consistent with recent work on language-model agents that
interleave reasoning and action, invoke external tools, browse Web resources,
and interact with software or embodied environments
\cite{Yao2023ReAct,Schick2023Toolformer,Nakano2021WebGPT,Wang2024Voyager,Yang2024SWEAgent}.

\subsection{Environment Access and Internalized Queries}

We distinguish two kinds of environments: deterministic and stochastic. In a
\emph{deterministic} environment, each action on a given state always yields
the same next state: a local file system, a code sandbox, or a fixed database
behaves this way. An environment becomes \emph{stochastic} once it includes
resources whose responses the oracle cannot predict or reproduce---for example,
a live Web service whose read may return different results on identical calls
and whose write may time out or fail, a resource other processes modify
concurrently, or a physical actuator with noise.

This paper assumes deterministic environments. Intuitively, an environment
state records the current condition of the task environment, such as files,
database rows, available documents, cached Web resources, permissions, and
runtime status. An action is a command issued to the environment, such as
deleting a file, applying a patch, or running a build. The transition function
specifies how actions change states, and the observation function specifies
what can be observed, such as a directory listing or a test report.

\begin{definition}[Deterministic Environment]
\label{def:env}
A \emph{deterministic environment} is a six-tuple
\[
  \Env = (E, A, O, f, g, e_0)
\]
with state set $E$, action set $A$, observation set $O$, transition function
$f : E \times A \to E$, observation function $g : E \to O$, and initial state
$e_0 \in E$. The function $f$ is total:
an action $u$ inapplicable in state $e$---a command that cannot be executed,
such as committing with nothing staged---is treated as a no-op, $f(e, u) = e$,
leaving the state unchanged.
\end{definition}

We use \emph{observe} and \emph{read} for obtaining information from the
environment, and \emph{act} and \emph{write} for issuing a command that may
change the environment state.

Since $f$ and $g$ are deterministic, an action sequence
$u_1,u_2,\ldots$ determines the environment trajectory
$e_0,e_1,\ldots$ by $e_i=f(e_{i-1},u_i)$ for $i\ge 1$; the environment
introduces no randomness of its own. With deterministic environments in place,
we next define agentic oracles.

\begin{definition}[Agentic Oracle]
\label{def:agent}
An \emph{agentic oracle} $\Oracle_A$ has a hidden state space $H$ and an initial hidden state
$h_0\in H$. It may have access to a task environment. If no environment is
present, write $e=\bot_E$ for a dummy environment state.

Its internal operation may exercise autonomous goal-directed control, may
access an environment in the read-only or read--write sense described above, or
both.

On a query $q\in\Sigma^*$, hidden state $h\in H$, and environment state $e$,
the oracle's internal operation induces a distribution $\Dist_{q,e}^h$
over internal transcripts $\Transcript$: explicit subqueries when they
occur, tool calls, embedded model or platform operations, intermediate states,
evaluations, revisions, and the stopping point at which the operation
terminates. Each transcript determines a final response $r\in\Sigma^*$ returned
at the oracle interface and a next hidden state $h'\in H$. The induced
distribution of $(r,h')$ is denoted by $\Pi_{q,h,e}$. The transcript
distribution and the induced distribution $\Pi_{q,h,e}$ are internal to $\Oracle_A$;
the oracle interface outputs only $r$.
\end{definition}

An SOTM whose oracle is an agentic oracle is called an \emph{agentic SOTM}.
Agentic oracles extend stationary stochastic oracles in the following sense. If
$H$ consists of a single hidden state, no changing environment state
affects the oracle except through information encoded in the query, and the
internal mechanism and resources of $\Oracle_A$ do not change from call to call
except through the query, then the internal transcript distribution depends
only on $q$.
Projecting the internal transcript onto the final response gives a fixed
query-dependent response distribution. In this special case, the external
behavior of $\Oracle_A$ is that of a stationary stochastic oracle. When hidden state
persists across calls, or when the accessible environment state changes in ways
not encoded in the query, the same query may induce different response
distributions at different times. Thus an agentic oracle is generally
non-stationary from the viewpoint of the PTM that queries it.

The preceding discussion separates two issues. One issue is what
response-generating capacity is available to an oracle. The other is how that
capacity is organized internally. For later token-cost comparisons, we use a
stationary stochastic oracle as a baseline capacity and compare it with an
agentic oracle that has the same capacity but internalizes some queries that
would otherwise pass through a query-response interface.

\begin{definition}[Internalized-Query Agentic Oracle]
\label{def:hardened-agentic-oracle}
Fix a stationary stochastic oracle $\Oracle$, viewed as the baseline
response-generating capacity. An agentic oracle $\Oracle_A$ is
\emph{internalized-query relative to $\Oracle$} if, for internal queries
corresponding to calls that could be made to $\Oracle$, it uses the same
response-generating capacity and response distributions as $\Oracle$, while
some internal queries that would otherwise be represented as query-response
calls to $\Oracle$ are internalized into the internal mechanism of
$\Oracle_A$.
\end{definition}

Internalized queries are what distinguish $\Oracle_A$ from an SOTM that uses
$\Oracle$ through an explicit query-response interface. An SOTM with
$\Oracle$ as its oracle may submit the same internal queries explicitly, but
those queries must cross the interface. By contrast, $\Oracle_A$ may generate,
store, or process internal task items, queries, responses, or state updates
inside its own mechanism without submitting them through a separate
query-response interface.

\subsection{Environment Tasks}

When no environment is present, we use the tasks reviewed in
Section~\ref{sec:prelim}. When environment access is part of the
computation, success is determined by the final output of the SOTM and the
environment state reached through the updates induced during the computation.

\begin{definition}[Environment Task]
\label{def:env-task}
An \emph{environment task} is a six-tuple $\Task=(X,Y,S,\Dist_X,\Kb,\Env)$, where
$X\subseteq\Sigma^*$ is the input space, $Y\subseteq\Sigma^*$ is the output
space, $\Dist_X$ is an input distribution on $X$, $\Kb$ is fixed background
information, $\Env=(E,A,O,f,g,e_0)$ is a deterministic environment, and
$S:X\times Y\times E\to[0,1]$ is a score function. Since $\Kb$ is fixed for the
task, any dependence of the score on the background information is treated as
part of the fixed score function $S$.
\end{definition}

For an environment task $\Task=(X,Y,S,\Dist_X,\Kb,\Env)$, an SOTM $\SOTM$
interacting with $\Env$ induces an environment trajectory
$(e_0,e_1,\ldots)$ through the actions issued during the computation. Let
$\sigma$ be the halting time at which the computation is evaluated; for
example, $\sigma$ may be a predetermined number of environment steps or may be
determined by the computation. If $\SOTM$ produces final output
$\widehat y\in Y$ at time $\sigma$, then $\SOTM$ \emph{achieves quality}
$\theta\in(0,1]$ on $\Task$ if
$\E[S(x,\widehat y,e_\sigma)]\geq \theta$, where the expectation follows the
expectation convention.

We also use the following notion for a finite sequence of updates in an
environment: for an environment task $\Task=(X,Y,S,\Dist_X,\Kb,\Env)$, an
\emph{$n$-stage sequential environment operation} is an interaction segment
of length $n$ within $\Task$ consisting of states, observations or task items,
actions, and transitions $(e_0,o_1,u_1,e_1,\ldots,o_n,u_n,e_n)$, where at each
stage $t=1,\ldots,n$, the segment starts from the current environment state,
the observation or task item $o_t$ is determined by the environment, the action
$u_t\in A$ is applied, and the next state $e_t$ is determined by the
environment transition rule.

For an instance $x\in X$ and terminal output $y\in Y$, the operation has
terminal score $S(x,y,e_n)$ and succeeds if $S(x,y,e_n)=1$.

Environment updates in an agentic SOTM may be caused by the directing PTM, and,
when the oracle has read--write environment access, by internal actions issued
by the agentic oracle.

\begin{definition}[Agentic SOTM with Environment Access]
\label{def:asotm}
An \emph{agentic SOTM with environment access} is an agentic SOTM
$\SOTM_A=(M_A,\Oracle_A)$ used to solve an environment task. The access may be
read-only, in which case the SOTM or the oracle observes environment
information without changing the environment state, or read--write, in which
case actions may update the environment state. The PTM $M_A$ can
observe the environment through $g:E\to O$ and issue actions in $A$,
after which the environment updates according to $f:E\times A\to E$. If
$\Oracle_A$ has read--write environment access, its internal operation may also
inspect the environment and issue actions in $A$. The formal environment
trajectory records the cumulative state changes caused by both the PTM's
actions and the actions issued by $\Oracle_A$.

On full input context $\mathcal I$, at step $t = 1,2,\dots$, $M_A$
reads the current observation $o_t=g(e_t)$, forms a query $q_t$ using its
task-input tape, background-input tapes, internal state, prior transcript, and
random tape, and receives $r_t$ from $\Oracle_A$. Based on its current
configuration, including $o_t$ and $r_t$, $M_A$ may issue an action
$u_t\in A$. During the same call, if
$\Oracle_A$ has read--write environment access, its internal operation may also
issue environment actions. The environment trajectory is updated by applying the
actions issued by $M_A$ and by $\Oracle_A$ according to $f$ in their realized
order. The PTM's random tape provides local randomness at no token cost.
For a predetermined number of environment steps or a stopping rule, let
$\sigma(\mathcal I)$ be the halting time at which the final output and reached
environment state are evaluated; when $\mathcal I$ is fixed, we write $\sigma$.
\end{definition}

\subsection{Token Cost and Token Complexity}

With the agentic SOTM in place, we distinguish the token cost visible to its PTM
from the token cost incurred inside the agentic oracle. For an SOTM
$\SOTM_A=(M_A,\Oracle_A)$ whose oracle is agentic, the visible query-response
token cost reviewed in Section~\ref{sec:prelim} is referred to as the
\emph{orchestration token cost}. We write
$\TOK_{\SOTM_A,t}^{\mathrm{orch}}(\mathcal I;\alpha,\beta)$ for the turn-$t$
orchestration token cost and
$\TOK_{\SOTM_A}^{\mathrm{orch}}(\mathcal I;\alpha,\beta)$ for the corresponding
total orchestration token cost on full input context $\mathcal I$.

Before producing the final response, the agentic oracle may spend additional
tokens in its internal operations: model calls, tool calls, retrieval,
summaries, retries, evaluations, or embedded platform mechanisms. These tokens
are not visible on the query-response tape of the PTM, but they are real
execution-time token costs incurred by the agentic oracle. This motivates the
notion of agentic token cost.

For an agentic oracle $\Oracle_A$ and an internal transcript
$\Transcript$, let $\TOK_{\Oracle_A}(\Transcript;\alpha,\beta)$ denote the
\emph{agentic token cost} accumulated by the oracle's internal operations while
producing a response. Conditional on query $q$, current hidden state $h$, and
current environment state $e$, its expected agentic token cost is
$\E_{\Transcript\sim\Dist_{q,e}^h}[\TOK_{\Oracle_A}(\Transcript;\alpha,\beta)]$,
where $\Transcript$ is drawn from the internal transcript distribution
$\Dist_{q,e}^h$ of Definition~\ref{def:agent}. That is, the expectation is
over the oracle's internal randomness and the outputs of any models, tools, or
internal modules it invokes.

The directing PTM generally does not know how many internal steps the agentic
oracle will take, which tools or models it will call, whether it will retry,
revise, summarize, or stop early, or how much hidden state it will carry
forward. Thus the agentic token cost is generated by the realized internal
trajectory, not by the visible query and final response alone. In commercial
agentic systems, such as OpenAI Codex or Claude Code, the user may not know,
even after the oracle returns its final response, how this realized cost is
computed or decomposed across internal prompts, model calls, tool calls,
retries, summaries, or embedded platform mechanisms.

For the $t$-th call of $\SOTM_A=(M_A,\Oracle_A)$ on $\mathcal I$, let
\[
  \TOK_{\SOTM_A,t}^{\mathrm{ag}}(\mathcal I;\alpha,\beta)
  =
  \TOK_{\Oracle_A}(\Transcript_t(\mathcal I);\alpha,\beta)
\]
be the agentic token cost of its internal transcript. The total token
cost of the call is
\begin{equation}
\label{eq:per-call-token-decomposition}
  \TOK_{\SOTM_A,t}(\mathcal I;\alpha,\beta)
  =
  \TOK_{\SOTM_A,t}^{\mathrm{orch}}(\mathcal I;\alpha,\beta)
  +
  \TOK_{\SOTM_A,t}^{\mathrm{ag}}(\mathcal I;\alpha,\beta).
\end{equation}
When $\SOTM_A$ and $\mathcal I$ are fixed, we may write
$\TOK_t^{\mathrm{orch}}$, $\TOK_t^{\mathrm{ag}}$, and $\TOK_t$.

\begin{definition}[Agentic Token Complexity]
\label{def:complexity}
Fix an environment task $T=(X,Y,S,\Dist_X,\Kb,\Env)$, where $\Kb$ may be empty,
and an agentic oracle $\Oracle_A$. Let $\Fset_T(\theta;\Oracle_A)$ be the set of
agentic SOTMs
$\SOTM_A=(M_A,\Oracle_A)$ that achieve quality $\theta$ on $T$, that is,
$\E[S(x,\widehat y,e_{\sigma(\mathcal I)})]\geq
\theta$, where $\widehat y$ is the final output at halting time
$\sigma(\mathcal I)$, with $\Kb$ fixed and $\mathcal I$ following the
full-input-context convention. Define
\begin{equation}
\label{eq:expected-total-agentic-token-cost}
  \Phi(\SOTM_A;\alpha,\beta)
  =
  \E\!\left[
    \sum_{t=1}^{\sigma(\mathcal I)}
    \TOK_{\SOTM_A,t}(\mathcal I;\alpha,\beta)
  \right],
\end{equation}
\[
  \Phi^{\mathrm{orch}}(\SOTM_A;\alpha,\beta)
  =
  \E\!\left[
    \sum_{t=1}^{\sigma(\mathcal I)}
    \TOK_{\SOTM_A,t}^{\mathrm{orch}}(\mathcal I;\alpha,\beta)
  \right],
  \qquad
  \Phi^{\mathrm{ag}}(\SOTM_A;\alpha,\beta)
  =
  \E\!\left[
    \sum_{t=1}^{\sigma(\mathcal I)}
    \TOK_{\SOTM_A,t}^{\mathrm{ag}}(\mathcal I;\alpha,\beta)
  \right],
\]
where all expectations follow the convention above and include the oracle's
where all expectations follow the expectation convention and include the oracle's
internal randomness. The \emph{agentic token complexity} of $T$ relative to
$\Oracle_A$ is the infimum of expected total token cost over
$\SOTM_A\in\Fset_T(\theta;\Oracle_A)$:
\[
  \kappa_T(\theta;\Oracle_A,\alpha,\beta)
  =
  \inf_{\SOTM_A \in \Fset_T(\theta;\Oracle_A)}
  \Phi(\SOTM_A;\alpha,\beta),
\]
and the component complexities are
\[
  \kappa_T^c(\theta;\Oracle_A,\alpha,\beta)
  =
  \inf_{\SOTM_A \in \Fset_T(\theta;\Oracle_A)}
  \Phi^c(\SOTM_A;\alpha,\beta),
  \qquad c\in\{\mathrm{orch},\mathrm{ag}\}.
\]
We write $\korch(\theta;\Oracle_A,\alpha,\beta)$ and
$\kag(\theta;\Oracle_A,\alpha,\beta)$ for the cases $c=\mathrm{orch}$ and
$c=\mathrm{ag}$, respectively.
\end{definition}

The infima above need not be attained: realized token costs are integer-valued,
but expected token costs need not be, and an optimizing SOTM may not exist.

\subsection{Token-Cost Analysis}
\label{sec:decomp}

Token costs are separated into orchestration token cost, which is visible
to the PTM, and agentic token cost, which is incurred inside the oracle. For a
fixed agentic SOTM and a fixed input context, the per-call token cost decomposes
into these two components. Token complexity, however, is obtained after taking
expectations and optimizing over all agentic SOTMs that achieve a target quality
level. The question is how this optimization interacts with the two token-cost
components. The first result gives a general lower bound in terms of the
component complexities.

Throughout this subsection, fix the environment task $T$, the agentic oracle
$\Oracle_A$, and token-cost parameters $\alpha,\beta$. We suppress these
arguments and write $\Fset(\theta)$, $\kappa_T(\theta)$, $\korch(\theta)$, and
$\kag(\theta)$ for the feasible set, token complexity, orchestration component
complexity, and agentic component complexity of
Definition~\ref{def:complexity}. The superscripts $\mathrm{orch}$ and
$\mathrm{ag}$ serve as the component labels.

\begin{proposition}[Decomposition Inequality]
\label{prop:decomp}
The token complexity and its two component complexities satisfy
\[
  \kappa_T(\theta) ~\geq~ \korch(\theta) + \kag(\theta).
\]
Equality holds if the two component infima can be approached simultaneously
by feasible agentic SOTMs. In particular, equality holds if a common feasible
agentic SOTM attains both component infima.
\end{proposition}

\begin{proof}
It follows from Equation~\eqref{eq:per-call-token-decomposition} and additivity
of expectation that, for every
$\SOTM_A \in \Fset(\theta)$,
$\Phi(\SOTM_A) = \Phi^{\mathrm{orch}}(\SOTM_A) +
\Phi^{\mathrm{ag}}(\SOTM_A) \geq \korch(\theta) + \kag(\theta)$, each term
being at least its infimum over the same feasible set.
Taking the infimum of the left side preserves the bound. Equality holds if
feasible agentic SOTMs can make
$\Phi^{\mathrm{orch}}$ approach $\korch(\theta)$ and $\Phi^{\mathrm{ag}}$
approach $\kag(\theta)$ simultaneously. A common minimizer gives the special
case in which both infima are attained.
\end{proof}

The component complexities do not necessarily sum to token complexity, because
the SOTM that minimizes orchestration token cost need not be the same SOTM that
minimizes agentic token cost. A strict gap can occur when sequences approaching
the two component infima are incompatible, so no feasible agentic SOTM can
approach both optima at once. The next result gives a simple case where one
component nevertheless controls token complexity up to a constant factor.

For any given agentic SOTM $\SOTM_A=(M_A,\Oracle_A)$ and any turn $t$ of its
computation, let $\mathcal F_t$ denote the information available after $M_A$ has
issued query $q_t$ but before $\Oracle_A$'s internal operation at that
turn is realized. Thus $\mathcal F_t$ consists of the task instance $x$, the
background materials $\Kb$, if present, the prior query-response transcript
$(q_1,r_1,\ldots,q_{t-1},r_{t-1})$, the state of $M_A$, the randomness revealed
up to turn $t$, and the issued query $q_t$. It does not include the internal
transcript or final response of $\Oracle_A$ at turn $t$.

\begin{proposition}[Dominant Component Bound]
\label{prop:dominant-component}
Let $\lambda_1>0$ and $\lambda_2\geq 0$.
The following hold.
\begin{enumerate}
  \item If, for every realized turn $t\leq\sigma$,
  $\E[\TOK_t^{\mathrm{ag}}\mid\mathcal F_t]\geq
  \lambda_1\E[\TOK_t^{\mathrm{orch}}\mid\mathcal F_t]$, then
  $\kag(\theta) \leq \kappa_T(\theta) \leq
  (1 + 1/\lambda_1)\kag(\theta)$.

  \item If, for every realized turn $t\leq\sigma$,
  $\E[\TOK_t^{\mathrm{ag}}\mid\mathcal F_t]\leq
  \lambda_2\E[\TOK_t^{\mathrm{orch}}\mid\mathcal F_t]$, then
  $\korch(\theta) \leq \kappa_T(\theta) \leq
  (1+\lambda_2)\korch(\theta)$.
\end{enumerate}
\end{proposition}

The first bound describes applications in which a compact query delegates
substantial internal work, so agentic token cost dominates orchestration token
cost. The second bound is mainly a consistency check. When $\lambda_2=0$,
the agentic token cost vanishes and $\kappa_T(\theta)=\korch(\theta)$,
recovering the stationary-oracle SOTM framework.

\begin{proof}
For item 1, applying the conditional bound at each realized turn $t\leq\sigma$
and using the tower property gives $\Phi^{\mathrm{ag}}(\SOTM_A) \geq
\lambda_1\,\Phi^{\mathrm{orch}}(\SOTM_A)$ for every feasible $\SOTM_A$. By the
definitions of $\Phi$, $\Phi^{\mathrm{orch}}$, and $\Phi^{\mathrm{ag}}$ in
Definition~\ref{def:complexity},
\[
  \Phi(\SOTM_A)
  =
  \Phi^{\mathrm{orch}}(\SOTM_A)+\Phi^{\mathrm{ag}}(\SOTM_A).
\]
The assumption gives
$\Phi^{\mathrm{orch}}(\SOTM_A)\leq
\Phi^{\mathrm{ag}}(\SOTM_A)/\lambda_1$, hence
$\Phi(\SOTM_A)\leq(1+1/\lambda_1)\Phi^{\mathrm{ag}}(\SOTM_A)$. Also, since
orchestration token cost is nonnegative,
$\Phi(\SOTM_A)\geq\Phi^{\mathrm{ag}}(\SOTM_A)$. Taking infima over
$\Fset(\theta)$ gives the first bound. Item 2 is symmetric with $\lambda_2$
in place of $1/\lambda_1$; $\lambda_2 = 0$ forces
$\Phi^{\mathrm{ag}} \equiv 0$, so $\Phi = \Phi^{\mathrm{orch}}$ and the
complexities coincide. The bound does not require a common minimizer, so it
holds even where Proposition~\ref{prop:decomp} is a strict inequality.
\end{proof}

Expected agentic token cost may depend on internal dispatch choices that are
not exposed to the PTM. Multi-model agentic systems may use member models either
as alternatives for the same internal query or as complementary specialists for
different subtasks. The following theorem addresses the first case and gives an
ordering rule in terms of the cost-success ratio $\TOK_i/p_i$, namely, mean
agentic token cost divided by success probability. Consider an agentic oracle
that can dispatch a query internally to member models $1,\ldots,m$, trying them
sequentially until one produces a usable response. In that setting, dispatch
order affects expected agentic token cost without changing the success
probability. The member outcomes are independent. Member $i$
produces a usable response with probability
$p_i\in[0,1]$, and invoking it has mean agentic token cost $\TOK_i$. An internal
verifier accepts exactly usable responses and, when at least one usable response
is available, selects one as the oracle response. Parallel dispatch to all
members and verified sequential dispatch over the same members have the same
success probability
\[
  p_{\mathrm{agg}} = 1-\prod_{i=1}^{m}(1-p_i).
\]
Thus the agentic token cost of sequential dispatch is not determined only by
the set of member models; it also depends on the order in which they are
queried.

\begin{theorem}[Ordering Rule for Verified Sequential Dispatch]
\label{thm:dispatch-ordering}
For an ordering $\pi$ of the member models, where $\pi(k)$ is the member queried
in the $k$-th position, the expected member-call contribution to agentic token
cost is
\begin{equation}
\label{eq:verified-sequential-dispatch-cost}
  \sum_{k=1}^{m} \TOK_{\pi(k)}
  \prod_{\ell<k}\bigl(1-p_{\pi(\ell)}\bigr),
\end{equation}
excluding fixed dispatch, verification, and aggregation overheads. For a fixed
set of members with $p_i>0$, this expected cost is minimized by ordering
members by nondecreasing ratio $\TOK_i/p_i$.
\end{theorem}

\begin{proof}
Both parallel and sequential dispatch fail exactly when every member fails to
produce a usable response, which gives $p_{\mathrm{agg}}$.
In a sequential order $\pi$, member $\pi(k)$ is invoked only if all earlier
members failed verification, an event of probability
$\prod_{\ell<k}(1-p_{\pi(\ell)})$, so linearity of expectation gives the
cost formula \eqref{eq:verified-sequential-dispatch-cost}. For the ordering
rule, compare two adjacent members $i$ and $j$.
Placing $i$ before $j$ costs $\TOK_i+(1-p_i)\TOK_j$ in expectation, while
placing $j$ before $i$ costs $\TOK_j+(1-p_j)\TOK_i$. The first order is no
worse exactly when $p_i\TOK_j \geq p_j\TOK_i$, equivalently
$\TOK_i/p_i \leq \TOK_j/p_j$. Repeated adjacent interchanges yield the stated
order.
\end{proof}

\paragraph{Agentic token costs are not externally estimable.}
For both stationary stochastic and agentic oracles, the internal dispatch
mechanism and member models are not exposed to the PTM that issues queries to
the oracle.
For agentic oracles, this opacity has a stronger consequence: during the
computation, the PTM generally has no way to estimate the token cost of the
oracle's internal operations from the query-response record visible to the PTM.
Those operations may include adaptive routing, retries, tool calls, model
invocations, summaries, verification steps, or discarded intermediate work. In
deployed agentic systems, even ex post information may be unavailable: billing
may report only an aggregate charge, partial usage, or no usable breakdown of
how tokens were spent across internal prompts, model calls, tool calls, retries,
summaries, verification steps, and other intermediate work. After the PTM
issues the query, it may have little control over the internal execution path
of the agentic oracle. From the caller's perspective, the agentic oracle can
therefore behave as an opaque token-cost sink. This is a drawback of computing
with agentic oracles: meaningful control of agentic token cost requires billing
transparency, usage telemetry, or an interface that lets the PTM influence
internal choices such as which member models are queried and in what order.

\subsection{Stationary-Oracle Simulation under Transparency}
\label{sec:transparent-simulation}

Under a strong transparency assumption, an agentic SOTM can be simulated by a
stationary-oracle SOTM. This result is a calibration rather than an advantage
claim: it identifies a setting in which the agentic oracle's implementation is
available, so the simulating SOTM can reproduce its behavior using the
stationary stochastic oracle.

We say that an agentic oracle $\Oracle_A$ is \emph{transparent over} a stationary
stochastic oracle $\Oracle$ if $\Oracle_A$ admits an SOTM implementation
$\SOTM'=(M',\Oracle)$ and the full description of $\SOTM'$ is available,
including its transition function, initial configuration, the form of queries
made by $M'$ to $\Oracle$, state representation, and output rule. The following
proposition makes the resulting simulation and its token-cost comparison
explicit.

\begin{proposition}[Agentic-SOTM Simulation under Transparency]
\label{prop:agentic-simulation}
Let $\SOTM_A=(M_A,\Oracle_A)$ be an agentic SOTM operating in a deterministic
environment. Suppose that $\Oracle_A$ is transparent over a stationary
stochastic oracle $\Oracle$. Then one can construct an SOTM
$\SOTM_S=(M_S,\Oracle)$ that induces the same distribution over environment
trajectories as $\SOTM_A$. Consequently, $\SOTM_S$ and $\SOTM_A$ achieve the
same quality on the corresponding environment task. Moreover, the simulation
satisfies
\begin{equation}
\label{eq:transparent-simulation-token-cost}
  \Phi(\SOTM_S)
  \leq
  \Phi^{\mathrm{orch}}(\SOTM_A)
  +
  \Phi^{\mathrm{ag}}(\SOTM_A).
\end{equation}
\end{proposition}

\begin{proof}
Because the full description of $\SOTM'=(M',\Oracle)$ is available, the
simulating PTM $M_S$ simulates $\SOTM'$ within its own computation. Whenever
the simulated $M_A$ queries $\Oracle_A$, $M_S$ runs $M'$ with the same query,
generates matching oracle-call outcomes and random choices for the simulated
execution of $\SOTM'$, and returns the generated response to the simulated
$M_A$. Induction over turns gives the same distribution over actions and
environment states, because the next configuration and environment state are
determined by the same simulated response and action history. Hence the two
SOTMs induce the same distribution over environment trajectories and achieve the
same quality.

For the token-cost inequality \eqref{eq:transparent-simulation-token-cost}, the
orchestration query-response turns between $M_A$ and $\Oracle_A$ are included in
the right-hand side through
$\Phi^{\mathrm{orch}}(\SOTM_A)$, but the simulator need not reproduce those
turns as oracle calls. Thus the inequality need not be an equality: those turns
may be handled as internal simulation steps of $M_S$, which carry no token cost
in the present execution-time token-cost model. Each oracle call made by $M'$ is
simulated by one query-response turn
between $M_S$ and $\Oracle$. By assumption, the tokens in those internal turns
are counted in the agentic token cost $\Phi^{\mathrm{ag}}(\SOTM_A)$.
Simulating the remaining PTM transitions incurs no token cost. Taking
expectations gives inequality~\eqref{eq:transparent-simulation-token-cost}.
\end{proof}

The transparency assumption is strong and often unrealistic for deployed
agentic systems. An agentic oracle may rely on domain-specific models, private
tools, proprietary retrieval systems, private memory stores, embedded model or
platform mechanisms, or other internal resources that cannot be reached through
the oracle access available to a stationary-oracle SOTM.

\section{Advantages of Agentic SOTMs}
\label{sec:advantages}

Despite the transparency-based simulation result in Section~\ref{sec:transparent-simulation},
agentic SOTMs remain attractive even though agentic token costs are not
externally estimable. Agentic SOTMs allow goal-directed work to be delegated to
an agentic oracle without requiring the PTM designer to specify every
intermediate step needed to achieve a solution. This can reduce coding effort,
let the designer focus on the crucial steps assigned to the PTM, improve
productivity, reduce labor cost, and shorten development time.

The question in this section is whether agentic SOTMs can also reduce
execution-time token cost relative to SOTMs with stationary stochastic oracles
for the same task and at the same quality level. The results below show that,
agentic SOTMs can do so by internalizing state, environment updates, or
intermediate operations that a stationary-oracle SOTM must repeatedly transmit
through the query-response interface.

\subsection{Token-Cost Advantages with Environment Updates}

One source of advantage is retained state during environment updates.
An agentic oracle may follow an internal task rule that uses the accumulated
environment-update history to choose intermediate actions and, eventually,
the final response. This rule and the state accumulated through the resulting
actions are not exposed at the query-response interface.

An SOTM using a stationary stochastic oracle can, in principle, mimic the effect
of such a rule by keeping a record of the prior computation on its own tapes and
retransmitting the relevant part of that record in later queries. The token-cost
gap arises because this record must repeatedly cross the query-response
interface for the stationary oracle, while the agentic oracle may retain the
corresponding state internally. Consider the following environment task, in
which the environment is updated through a sequence of actions.

\paragraph{The retained-history environment task.}

Let $\mathcal V$ be the token vocabulary for the underlying tokenizer. Fix an
integer $K\geq 2$ and the action set $A=\{u_1,\ldots,u_K\}$. The input space $X$
consists of tuples $x=(m,\mathbf n,\mathbf B)$, where $m$ is a positive
integer, $\mathbf n=(n_1,\ldots,n_m)$ is a vector of positive integers, and
$\mathbf B=(b_{i,t})_{1\leq i\leq m,\ 1\leq t\leq n_i}$ is an array of positive
integers. Here $n_i$ is the number of state updates in phase $i$, and
$b_{i,t}$ is the token length of the observed token string at update $t$ of
phase $i$. Under $\Dist_X$, first draw $m$, then draw $n_1,\ldots,n_m$, and
then draw each $b_{i,t}$, all independently according to
$\Pr[N=k]=6/(\pi^2 k^2)$, $k=1,2,\ldots$.

The \emph{retained-history environment task} is the six-tuple
\[
  T_{\mathrm{RH}}^{\mathrm{env}}=(X,Y,S,\Dist_X,\Kb,\Env).
\]
The background information $\Kb$ is arbitrary but fixed; it may be empty. Since
$\Kb$ is fixed across the computation, it does not affect the repeated-prefix
token-cost comparison below. The output space $Y$ consists of terminal
certificates of completion or failure. On input $x=(m,\mathbf n,\mathbf B)$,
the deterministic environment $\Env=(E,A,O,f,g,e_0)$ is organized into $m$
phases. Phase $i$ consists of $n_i$ state updates.

At update $t=1,\ldots,n_i$ of phase $i$, the observation available from the
environment includes the next observed token string
$z_{i,t}\in\mathcal V^{b_{i,t}}$ of exactly $b_{i,t}$ tokens. The prefix
available by update $t$ of phase $i$ is
$(z_{i,1},\ldots,z_{i,t})$. For each prefix, the environment has a unique
correct action in $A$ for that update. The transition function $f$ moves the
environment to a failure state after the first incorrect action and to a goal
state after all phases are completed correctly. The score function
$S:X\times Y\times E\to[0,1]$ assigns value $1$ exactly when the final output
and reached environment state certify completion of all phases, and assigns
value $0$ otherwise. This completes the specification of the retained-history
environment task.

For the token-cost comparison, we choose the correct-action functions randomly
when constructing the deterministic environment. For each input $x$, write
the correct action at update $t$ of phase $i$ as
$u_{i,t}=F^x_{i,t}(z_{i,1},\ldots,z_{i,t})$, where
$F^x_{i,t}:\mathcal V^{b_{i,1}}\times\cdots\times\mathcal V^{b_{i,t}}\to A$,
and choose the functions $F^x_{i,t}$ independently and uniformly from all such
functions. These functions determine the relevant part of the environment
transition function, but they are not placed on the PTM's input or
background tapes. We want to know what token cost is incurred when an agentic
oracle can retain each phase history internally, while a stationary-oracle SOTM
must transmit the relevant prefix through the query-response interface whenever
it queries the stationary oracle.

Because the action-selection functions are uniform over all prefix functions, a
representation that does not determine the full prefix gives no better than
$1/K$ success probability for the next action.

An agentic oracle $\Oracle_A$ is \emph{retained-history capable} for
$T_{\mathrm{RH}}^{\mathrm{env}}$ relative to a stationary stochastic oracle
$\Oracle$ if $\Oracle_A$ can interact with the task environment, retains each
phase history internally, and is internalized-query relative to $\Oracle$, where
$\Oracle$ provides the response-generating capacity used to identify the correct
actions for the required environment updates.

\begin{theorem}[Token Costs with Retained Environment History]
\label{thm:persistent-state-separation}
Suppose $\Oracle_A$ is retained-history capable for
$T_{\mathrm{RH}}^{\mathrm{env}}$ relative to a stationary stochastic oracle
$\Oracle$. Then we can construct an agentic SOTM $\SOTM_A=(M_A,\Oracle_A)$ such
that $\SOTM_A$ achieves quality $1$ on $T_{\mathrm{RH}}^{\mathrm{env}}$,
meaning that its expected score on the environment task is $1$.
For each phase $i$, let $q_i$ and $r_i$ be the visible query and final response
on $M_A$'s query-response tape for the delegated call corresponding to phase
$i$, and let $\ell_{i,t}$ denote the token length of the common intermediate
action representation used at update $t$ of that phase. Then the
delegated call at turn $i$ has expected total token cost, including
orchestration and agentic token costs, at most
\[
  \alpha|\tau(q_i)|+\beta|\tau(r_i)|
  +
  \alpha\sum_{t=1}^{n_i}b_{i,t}
  +
  \beta\sum_{t=1}^{n_i}\ell_{i,t}.
\]
Conversely, for any $\theta$ with $1/K<\theta\leq 1$, let $\SOTM$ be any SOTM
using $\Oracle$ whose PTM has no access to the phase action-selection functions
$F_i^x=(F^x_{i,1},\ldots,F^x_{i,n_i})$ except through queries to $\Oracle$, and
that solves the same environment task with quality at least $\theta$. Then its
expected intermediate query-response token cost for phase $i$, using $\Oracle$,
is at least
\[
  \frac{\theta-1/K}{1-1/K}
  \left(
    \alpha\sum_{t=1}^{n_i}\sum_{s=1}^{t} b_{i,s}
    +
    \beta\sum_{t=1}^{n_i}\ell_{i,t}
  \right).
\]
\end{theorem}

\begin{proof}
Construct $M_A$ as follows. On input $x=(m,\mathbf n,\mathbf B)$, for each
phase $i$ of the task $T_{\mathrm{RH}}^{\mathrm{env}}$, $M_A$ issues a query
$q_i$ that instructs $\Oracle_A$ to carry out
the $n_i$ environment updates of phase $i$. Let $r_i$ be
the final response that $\Oracle_A$ returns on $M_A$'s query-response tape in
response to query $q_i$. After all phases are completed, $M_A$ outputs a
terminal certificate. Since $\Oracle_A$ can interact with the task
environment, retains each phase history internally, and is internalized-query relative to
$\Oracle$, it can use $\Oracle$'s response-generating capacity while
avoiding retransmission of the retained phase history. Thus the constructed
agentic SOTM has expected score $1$, and hence achieves quality $1$.

For the token-cost upper bound of $\SOTM_A$ at turn $i$, focus on the delegated
call at that turn; other queries, if any, are not part of this per-turn
comparison. The orchestration token cost of this delegated call is
$\alpha|\tau(q_i)|+\beta|\tau(r_i)|$. Since each token string in phase $i$ is
processed once and its resulting internal state can be used in later internal
stages of that phase without reprocessing that token string, the agentic
token cost attributable to this call is at most
\[
  \alpha\sum_{t=1}^{n_i}b_{i,t}
  +
  \beta\sum_{t=1}^{n_i}\ell_{i,t}.
\]
Thus the expected total token cost of the delegated call, including
orchestration and agentic token costs, is at most
\[
  \alpha|\tau(q_i)|+\beta|\tau(r_i)|
  +
  \alpha\sum_{t=1}^{n_i}b_{i,t}
  +
  \beta\sum_{t=1}^{n_i}\ell_{i,t}.
\]

For the token-cost lower bound of an SOTM using $\Oracle$ during phase $i$, let
$E$ be the event that the SOTM using $\Oracle$ supplies the exact prefix
$(z_{i,1},\ldots,z_{i,t})$ at every update $t$ of phase $i$. If $E$ fails, then
there is a first update $t$ in phase $i$ at which
the SOTM does not supply the exact prefix $(z_{i,1},\ldots,z_{i,t})$ to
$\Oracle$. Under the distribution over the action-selection functions, and
because $A=\{u_1,\ldots,u_K\}$, the correct action for the omitted prefix is
uniformly distributed over $A$ and is independent of the PTM's available
information. Therefore the PTM's selected action is correct with probability at
most $1/K$. It follows that
\[
\begin{aligned}
  \Pr[\text{success}]
  &\leq \Pr[E]+\frac{1}{K}\Pr[\neg E]  \\
  &=\frac{1}{K}+\left(1-\frac{1}{K}\right)\Pr[E].
\end{aligned}
\]
For this task, the score is binary: success has score $1$ and failure has score
$0$. Hence the expected score equals $\Pr[\text{success}]$. Achieving quality
at least $\theta$ therefore gives $\Pr[\text{success}]\geq\theta$. Combining
this with the preceding inequality gives
\[
  \theta
  \leq
  \Pr[\text{success}]
  \leq
  \frac{1}{K}+\left(1-\frac{1}{K}\right)\Pr[E],
\]
and hence
\[
  \Pr[E]\geq \frac{\theta-1/K}{1-1/K}.
\]

On the event $E$, the SOTM using $\Oracle$ supplies the exact prefix at every
update of phase $i$. Thus, at update $t$, its query to $\Oracle$ must include
the token strings $z_{i,1},\ldots,z_{i,t}$. Since $z_{i,s}$ has $b_{i,s}$
tokens, the total observed-token length of this prefix is
$\sum_{s=1}^{t}b_{i,s}$. The corresponding intermediate action representation
has $\ell_{i,t}$ tokens. Summing the query and action-representation token costs
over $t=1,\ldots,n_i$, the intermediate oracle token cost during phase $i$ is
at least the quantity $C_i^{\mathrm{stat}}$ defined below. The double sum appears
because the accumulated prefix must be retransmitted at each update; the
stationary oracle itself does not retain the intermediate state from earlier
calls.
This lower bound counts only the intermediate stationary-oracle calls needed to
carry out phase $i$. It does not include any initial instruction or final-report
tokens, so omitting such tokens only weakens the lower bound.
Let
\[
  C_i^{\mathrm{stat}}
  =
  \alpha\sum_{t=1}^{n_i}\sum_{s=1}^{t} b_{i,s}
  +
  \beta\sum_{t=1}^{n_i}\ell_{i,t}.
\]
Let $\widehat C_i^{\mathrm{stat}}$ be the realized intermediate query-response
token cost incurred by the SOTM using $\Oracle$ while carrying out phase $i$.
On $E$, this realized cost is at least $C_i^{\mathrm{stat}}$; outside $E$, it is
nonnegative. Thus
\[
  \widehat C_i^{\mathrm{stat}}
  \geq
  C_i^{\mathrm{stat}}\mathbf 1_E.
\]
Taking expectations gives
\[
  \E[\widehat C_i^{\mathrm{stat}}]
  \geq
  \Pr[E]\,C_i^{\mathrm{stat}}
  \geq
  \frac{\theta-1/K}{1-1/K}C_i^{\mathrm{stat}}.
\]
This completes the proof.
\end{proof}

When the observed-token lengths in phase $i$ are bounded above and below by
positive constants independent of $n_i$, and the visible query and final
response lengths are independent of $n_i$,
Theorem~\ref{thm:persistent-state-separation} shows a token-cost gap that grows
linearly with $n_i$ for the updates in phase $i$. Thus, even when both systems
meet the same target quality on the same environment task, the agentic SOTM can
use fewer execution-time tokens: the intermediate state needed across updates
can remain inside the agentic oracle, whereas the stationary-oracle SOTM must
repeatedly transmit the relevant accumulated history through the query-response
interface.

\subsection{Token-Cost Advantages without Environment Access}

The token-cost advantage from retained intermediate state does not require
environment access. We call a task \emph{environment-isolated} if the
task provides only the given input and fixed background information,
with no environment access. For such a task, the multi-step computation
needed to produce the final output may occur entirely inside the agentic oracle:
decomposition, subquery generation, aggregation, evaluation, revision, and
repetition may all be internal steps. In this setting, an agentic oracle may
retain intermediate state across these internal steps, while an SOTM using a
stationary stochastic oracle must expose the relevant state in later queries.

Let $\Oracle$ be a stationary stochastic oracle. We say that an
$\Oracle$-internalized-query agentic oracle has a \emph{retained internal operation}
for an environment-isolated task if, after receiving a visible oracle query
$q_i$, it carries out an $n_i$-stage internal computation. The computation has
internal input token strings of lengths $b_{i,t}$, $t=1,\ldots,n_i$, and
generated intermediate token strings of lengths $\ell_{i,t}$. It uses $\Oracle$'s
response-generating capacity while retaining the internal state generated at
earlier internal stages without resubmitting it through a query-response
interface.

The retained internal operation is \emph{correct} if its final response yields
score $1$ for the environment-isolated task.

We state the environment-isolated version separately because it shows that
environment access is not necessary for retained-state token-cost
advantages. The same retained-history construction gives the following
environment-isolated analog of Theorem~\ref{thm:persistent-state-separation}.

\begin{theorem}[Token Costs with Retained Internal State]
\label{thm:environment-free-separation}
Let $\Oracle$ be a stationary stochastic oracle. There exists an
environment-isolated task and an agentic SOTM whose agentic oracle is
$\Oracle$-internalized-query and has a correct retained internal operation at
turn $i$, such that the agentic SOTM achieves quality $1$ on this task. If
$q_i$ is the visible oracle query at turn $i$ and $r_i$ is the final oracle
response on the PTM's query-response tape, then the delegated call at turn $i$
has expected total token cost, including orchestration and agentic token costs,
at most
\[
  \alpha|\tau(q_i)|+\beta|\tau(r_i)|
  +
  \alpha\sum_{t=1}^{n_i}b_{i,t}
  +
  \beta\sum_{t=1}^{n_i}\ell_{i,t}.
\]
Conversely, for every quality level $\theta$ with $1/K<\theta\leq 1$,
any SOTM using $\Oracle$ that solves the same environment-isolated task with quality at
least $\theta$ incurs expected intermediate query-response token cost at
least
\[
  \frac{\theta-1/K}{1-1/K}
  \left(
    \alpha\sum_{t=1}^{n_i}\sum_{s=1}^{t}b_{i,s}
    +
    \beta\sum_{t=1}^{n_i}\ell_{i,t}
  \right).
\]
\end{theorem}

\begin{proof}
The proof is the same as the proof of
Theorem~\ref{thm:persistent-state-separation}, except that the retained state is
internal to the agentic oracle rather than stored in an external environment.
The same prefix-randomization argument applies to the internal stages.
\end{proof}

Theorems~\ref{thm:persistent-state-separation}
and~\ref{thm:environment-free-separation} treat retained state as available
without additional token cost after it has been processed once. A more
conservative model may assign a reduced but nonzero token cost to reusing
retained state. The token-cost advantage can be parameterized by that reuse
cost.

\begin{proposition}[Token Cost with Retained-State Reuse]
\label{prop:retained-state-cost}
Suppose that processing each previously retained token costs a fraction
$\rho\in[0,1]$ of processing a newly supplied token. In either the
environment-update setting of Theorem~\ref{thm:persistent-state-separation}
or the environment-isolated setting of
Theorem~\ref{thm:environment-free-separation}, for a delegated call with visible
oracle query $q_i$, final oracle response $r_i$, and $n_i$ internal stages, the
corresponding agentic SOTM has expected total token cost at most
\[
  \alpha|\tau(q_i)|+\beta|\tau(r_i)|
  +
  \alpha\sum_{t=1}^{n_i}b_{i,t}
  +
  \beta\sum_{t=1}^{n_i}\ell_{i,t}
  +
  \rho\alpha\sum_{t=1}^{n_i}\sum_{s=1}^{t-1}b_{i,s},
\]
while, for any $\theta$ with $1/K<\theta\leq 1$, the stationary-oracle lower
bound remains
\[
  \frac{\theta-1/K}{1-1/K}
  \left(
    \alpha\sum_{t=1}^{n_i}\sum_{s=1}^{t}b_{i,s}
    +
    \beta\sum_{t=1}^{n_i}\ell_{i,t}
  \right).
\]
Thus the retained-state advantage depends on the retained-token processing
fraction $\rho$: when $\rho=0$, the ratio between the stationary-oracle lower
bound and the agentic upper bound grows linearly with $n_i$; for fixed
$\rho\in(0,1)$, the comparison may reduce to a constant-factor advantage.
\end{proposition}

\begin{proof}
The visible orchestration cost and the agentic cost of processing new token
strings contribute at most
$\alpha|\tau(q_i)|+\beta|\tau(r_i)|
+\alpha\sum_{t=1}^{n_i}b_{i,t}
+\beta\sum_{t=1}^{n_i}\ell_{i,t}$ in token cost. At internal stage $t$, the
agentic oracle may reprocess at most $\sum_{s=1}^{t-1}b_{i,s}$ retained
internal tokens, whose cost is scaled by $\rho$.
Therefore the retained-state cost is at most
\[
  \rho\alpha\sum_{t=1}^{n_i}\sum_{s=1}^{t-1}b_{i,s}.
\]
Adding this retained-token cost to the baseline cost gives the upper bound. The
stationary oracle itself retains no state, so the lower-bound argument used in
Theorems~\ref{thm:persistent-state-separation}
and~\ref{thm:environment-free-separation} is unchanged.
\end{proof}

The retained-state token-cost advantage captures one benefit of agentic
operation: information accumulated during the computation of a task need not
repeatedly cross the query-response tape. This token-cost analysis does not
capture a separate practical advantage of agentic systems. A user may specify a
goal and rely on the agentic oracle to construct and execute a control
procedure. By contrast, an SOTM using a stationary stochastic oracle may require
the directing PTM to contain
task-specific control code. In practice, producing that directing code can
involve skilled engineering labor, debugging, maintenance, and exploratory runs.
The present token-complexity framework assigns no cost to the description or
local computation of the PTM, so this reduction in development effort is not a
token-complexity result. Capturing it would require an additional measure of
task-specific directing-program, description, or development complexity.

\section{The Cost of Irreversible Actions}
\label{sec:foreclose}

Section~\ref{sec:advantages} focused on token-cost advantages of agentic SOTMs.
Those advantages arise from giving the agentic oracle more control over
intermediate state and, when read--write environment access is present, over
actions that update the environment. This delegated control over actions also creates a
risk that is not central in the stationary stochastic-oracle framework: an
action may change the environment in a way that makes the goal unachievable.
This section studies that risk. The main object is goal loss: a transition from
a state where the goal remains achievable to a state where it is no longer
achievable.

We first fix the set of states from which success can still be certified.
Throughout this section, for an environment task
$\Task=(X,Y,S,\Dist_X,\Kb,\Env)$ and a fixed input $x\in X$, write
$G_x=\{e\in E:\exists y\in Y\text{ such that }S(x,y,e)=1\}$ for the induced
set of goal states, and write $G=G_x$ when $x$ is fixed.

\subsection{Goal Loss and Avoidance}
\label{subsec:goal-loss-avoidance}

To formalize goal loss, we first investigate whether the goal remains
achievable from a given environment state. The PTM or agentic oracle may not be
able to determine this during execution, but it lets us distinguish actions that
preserve the possibility of success from actions after which success is no
longer achievable.

For a state $e$, let $R(e) = 1$ if some finite action sequence, applied through
the environment transition function $f$, leads from $e$ into $G$, and let
$R(e) = 0$ otherwise. Thus $R(e)$ is the reachability indicator for state $e$.
An action $u$ at a state $e$ with $R(e) = 1$ is \emph{goal-preserving} if
$R(f(e,u)) = 1$, and \emph{goal-losing} if $R(f(e,u)) = 0$.
We say that \emph{goal loss} occurs at step $t$ if the action taken at that step
moves the environment from a state where the goal is reachable to a state where
it is no longer reachable:
\[
  L_t=\{R(e_{t-1}) = 1, R(e_t) = 0\}, \qquad t\ge 1.
\]
We write $L_t^c$ for the complement of $L_t$, the event that goal loss does not
occur at step $t$.
The \emph{goal-loss time} is $t^\star = \min\{t : L_t\text{ occurs}\}$, with
$t^\star = \infty$ if no such $t$ exists.

Because goal loss is defined through the theoretical reachability indicator
$R$, the next result provides a structural criterion, not a test that the SOTM
can necessarily compute during execution.

For an SOTM $\SOTM$ interacting with the environment, define its
\emph{conditional per-step goal-loss probability} by
\[
  \pi_t =
  \Pr_{\SOTM}[L_t \mid L_1^c,\ldots,L_{t-1}^c].
\]
If $\Pr_{\SOTM}[L_1^c\cap\cdots\cap L_{t-1}^c]=0$, then the probability of no
goal loss up to the previous step is already zero, so later conditional
probabilities do not affect the conclusion.

\begin{theorem}[Goal-Loss Avoidance Criterion]
\label{thm:goal-loss-avoidance}
For any SOTM $\SOTM$ interacting with the environment, let $\pi_t$ be its
conditional per-step goal-loss probability. The following statements hold.
\begin{enumerate}
\item The probability that no goal loss occurs up to any finite time is the
product of the stepwise no-goal-loss probabilities: for every positive integer
$H$,
$\Pr_{\SOTM}\!\left[\bigcap_{t=1}^{H}L_t^c\right] =
\prod_{t=1}^{H}(1 - \pi_t)$.
\item The probability that no goal loss ever occurs is determined by the series
$\sum_{t\ge 1} \pi_t$: if $\sum_{t\ge 1} \pi_t < \infty$ and $\pi_t<1$ for all
$t$, then
$\Pr_{\SOTM}[t^\star = \infty] = \prod_{t \geq 1}(1-\pi_t) > 0$; if
$\sum_{t\ge 1} \pi_t = \infty$ or $\pi_t=1$ for some $t$, then
$\Pr_{\SOTM}[t^\star = \infty] = 0$ and
$\Pr_{\SOTM}\!\left[\bigcap_{t=1}^{H}L_t^c\right] \to 0$.
\item If the computation is evaluated after $H$ environment steps with score
$\mathbf 1_{\{e_H\in G\}}$, then any SOTM achieving quality at least $\theta$
satisfies
$\theta \leq \Pr_{\SOTM}[e_H \in G] \leq
\prod_{t=1}^{H}(1-\pi_t)$. If quality $\theta>0$ must be maintained over
arbitrarily large evaluation times, then the SOTM must have
$\sum_{t\ge 1} \pi_t < \infty$ and $\pi_t<1$ for all $t$.
\end{enumerate}
\end{theorem}

\begin{proof}
We prove the three claims in order.
\begin{enumerate}
\item The event that goal loss has not occurred by time $H$ is
$L_1^c\cap\cdots\cap L_H^c$. By the chain rule for conditional probabilities,
\[
  \Pr_{\SOTM}\!\left[\bigcap_{t=1}^{H}L_t^c\right]
  =
  \prod_{t=1}^{H}
  \Pr_{\SOTM}[L_t^c\mid L_1^c,\ldots,L_{t-1}^c]
  =
  \prod_{t=1}^{H}(1-\pi_t).
\]

\item As $H$ increases, the finite products
$\prod_{t=1}^{H}(1-\pi_t)$ decrease to the infinite product
$\prod_{t\geq 1}(1-\pi_t)$, which equals
$\Pr_{\SOTM}[t^\star=\infty]$ because the events
$\bigcap_{t=1}^{H}L_t^c$ decrease to the event $\{t^\star=\infty\}$. It remains
to determine when this infinite product is positive. If $\pi_t=1$ for some $t$,
then the product is zero. Thus the product can be positive only when
$0\leq \pi_t<1$ for all $t$. If infinitely many $\pi_t$ satisfy
$\pi_t>1/2$, then $\sum_{t\ge 1}\pi_t=\infty$ and the product is zero.
Otherwise, all but finitely many $\pi_t$ lie in $[0,1/2]$. We use the
elementary inequalities, for $0\leq x\leq 1/2$,
\[
  -2x \leq \ln(1-x) \leq -x.
\]
The inequality $\ln(1-x)\leq -x$ holds for all $x\in[0,1)$.
For the other inequality, let $h(x)=\ln(1-x)+2x$. Then $h(0)=0$ and
$h'(x)=2-1/(1-x)\geq 0$ for $0\leq x\leq 1/2$, so $h(x)\geq 0$ on this
interval.
Hence, when $\pi_t<1$ for all $t$, the tail sum
$\sum_t \ln(1-\pi_t)$ converges to a finite value exactly when
$\sum_{t\ge 1}\pi_t<\infty$; if $\sum_{t\ge 1}\pi_t=\infty$, then
$\sum_t \ln(1-\pi_t)=-\infty$. Exponentiating gives
$\prod_t(1-\pi_t)>0$ exactly when $\sum_{t\ge 1}\pi_t<\infty$ and $\pi_t<1$ for
all $t$, and gives product zero otherwise. This proves item 2.

\item When the computation is evaluated after $H$ environment steps with score
$\mathbf 1_{\{e_H\in G\}}$, expected quality equals
$\Pr_{\SOTM}[e_H\in G]$, so quality at least $\theta$ gives the first
inequality. Achieving score $1$ at that evaluation time requires that no goal
loss has occurred during steps $1,\ldots,H$. Hence
$\{e_H \in G\}$ is contained in $L_1^c\cap\cdots\cap L_H^c$, giving the second
inequality by item 1. If quality
$\theta>0$ must be maintained over arbitrarily large evaluation times, then
item 2 implies that this is possible only when
$\sum_{t\ge 1}\pi_t<\infty$ and $\pi_t<1$ for all $t$.
\end{enumerate}
\end{proof}

\subsection{Progress, Retry, and Token Cost}
\label{subsec:progress-retry-cost}

The avoidance criterion identifies how accumulated goal-loss risk constrains
success. We next quantify a common finite-stage pattern: at each stage, an
agentic oracle call may produce a progress event, cause goal loss, or make no
progress while retry remains possible.

Suppose reaching $G$ requires $L$ ordered progress events. At stage $j$, the
process is waiting for the $j$-th progress event. Let $p_j$ be the probability
that a call produces the $j$-th progress event, let $\pi_j$ be the probability
that the call leads to an action causing goal loss, and let $1-p_j-\pi_j$ be
the probability of no progress while retry remains possible, with $p_j>0$,
$\pi_j\geq0$, and $p_j+\pi_j\leq1$. The last condition ensures that the retry
probability is nonnegative. Conditional on the process being at stage $j$,
assume the successive call outcomes and token costs are i.i.d.\ until progress
or goal loss.
We call this setup the \emph{progress--retry--goal-loss model}.

In deployed agentic systems, the classification of a call as progress, no
progress while retry remains possible, or goal loss may itself not be externally
observable. The formulas therefore give conditional token-cost expressions once
such stagewise probabilities have been specified.

\begin{theorem}[Progress--Retry--Goal-Loss Formulas]
\label{thm:progress-retry}
Let $\SOTM_A$ be an agentic SOTM whose calls to the agentic oracle follow the
progress--retry--goal-loss model. Let $\sigma$ be the halting time.
Then the probability that $\SOTM_A$ reaches $G$ is
\[
  \Pr_{\SOTM_A}[e_\sigma \in G]
  ~=~
  \prod_{j=1}^{L} \frac{p_j}{p_j+\pi_j}.
\]
Moreover, let $\TOK_{\SOTM_A,j}(\alpha,\beta)$ be the expected token cost of one
stage-$j$ call. Let $N$ be the random variable representing the number of calls
made before the process halts, either by goal loss or by success at $G$. Then
\begin{equation}
\label{eq:progress-retry-calls}
  \E_{\SOTM_A}[N]
  ~=~
  \sum_{j=1}^{L}
  \left(\prod_{\ell<j}\frac{p_\ell}{p_\ell+\pi_\ell}\right)
  \frac{1}{p_j+\pi_j},
\end{equation}
and
\begin{equation}
\label{eq:progress-retry-token-cost}
  \TOK_{\SOTM_A}(\alpha,\beta)
  ~=~
  \sum_{j=1}^{L}
  \left(\prod_{\ell<j}\frac{p_\ell}{p_\ell+\pi_\ell}\right)
  \frac{\TOK_{\SOTM_A,j}(\alpha,\beta)}{p_j+\pi_j}.
\end{equation}
\end{theorem}

\begin{proof}
At stage $j$, each call either produces the $j$-th progress event with
probability $p_j$, causes goal loss with probability $\pi_j$, or produces no
progress while retry remains possible with probability $1-p_j-\pi_j$. If
progress and goal loss are grouped as the event that a call leaves stage $j$,
then this event has probability $p_j+\pi_j$ on each call. Conditional on
reaching stage $j$, the number of calls at that stage is the number needed until
the first progress-or-goal-loss event occurs. It is therefore a geometric
random variable with parameter $p_j+\pi_j$ and mean $1/(p_j+\pi_j)$.

Conditional on leaving stage $j$, progress occurs with
probability $p_j/(p_j+\pi_j)$. Thus $\SOTM_A$ reaches $G$ only if, at each
stage $j=1,\ldots,L$, the first non-retry outcome is progress rather than goal
loss. Since the stage outcomes are conditionally independent under the model,
the probability of this event is
\[
  \Pr_{\SOTM_A}[e_\sigma \in G]
  =
  \prod_{j=1}^{L}\frac{p_j}{p_j+\pi_j}.
\]

The probability that stage $j$ is reached is
\[
  \prod_{\ell<j}\frac{p_\ell}{p_\ell+\pi_\ell},
\]
because reaching stage $j$ requires that, for every earlier stage $\ell<j$, the
first non-retry outcome was progress rather than goal loss. Although the
process may stop before stage $L$, the expectation is computed by summing over
the fixed set of possible stages $1,\ldots,L$. A stage contributes only on the
event that it is reached. Multiplying this reach probability by the conditional
expected number of calls at stage $j$ and summing over $j$ gives
equality \eqref{eq:progress-retry-calls}.

For the expected token cost, condition again on reaching stage $j$. Given that
stage $j$ is reached, the expected number of calls made at that stage is
$1/(p_j+\pi_j)$, and each call has expected token cost
$\TOK_{\SOTM_A,j}(\alpha,\beta)$. Hence the expected token cost contributed by
stage $j$ is $\TOK_{\SOTM_A,j}(\alpha,\beta)/(p_j+\pi_j)$. Multiplying by the
probability of reaching stage $j$ and summing over $j$ gives
equality \eqref{eq:progress-retry-token-cost}. This completes the proof.
\end{proof}

A useful special case occurs when the stage parameters are the same at every
stage.

\begin{corollary}[Homogeneous Progress--Retry--Goal-Loss Formulas]
\label{cor:homogeneous-progress-retry-goal-loss}
Suppose, in the setting of Theorem~\ref{thm:progress-retry}, that the stage
parameters are homogeneous: for every $j=1,\ldots,L$, $p_j=p$,
$\pi_j=\pi$, and
$\TOK_{\SOTM_A,j}(\alpha,\beta)=\TOK_{\SOTM_A,1}(\alpha,\beta)$.
Then
\[
  \Pr_{\SOTM_A}[e_\sigma \in G] = \left(\frac{p}{p+\pi}\right)^L.
\]
Moreover,
\[
  \E_{\SOTM_A}[N] =
  \begin{cases}
    \displaystyle \frac{1-\left(\frac{p}{p+\pi}\right)^L}{\pi}, & \pi>0,\\[1.2ex]
    \displaystyle \frac{L}{p}, & \pi=0,
  \end{cases}
  \qquad
  \text{and}
  \qquad
  \TOK_{\SOTM_A}(\alpha,\beta)
  =
  \E_{\SOTM_A}[N]\,\TOK_{\SOTM_A,1}(\alpha,\beta).
\]
\end{corollary}

\subsection{Goal Depth, Reversibility, and Unavoidable Risk}
\label{subsec:depth-reversible}

The progress--retry--goal-loss formulas in
\eqref{eq:progress-retry-calls} and \eqref{eq:progress-retry-token-cost} give
exact expressions when the probabilities of progress, no progress with retry,
and goal loss are specified at each stage. A more structural lower bound comes
from the number of actions that any successful trajectory must take. If the goal
cannot be reached in fewer than $L$ environment-updating actions, then $L$ is
the \emph{goal depth}; every quality-$\theta$ computation must pay for at least
those actions on the event of success.
For this purpose, let $\TOK_{\min}(\alpha,\beta)>0$ denote a lower bound on the
conditional expected token cost of each oracle call, given any computation
history under consideration.

\begin{theorem}[Goal-Depth Lower Bound]
\label{thm:depth}
Let $L$ be the smallest integer $k$ for which there exist actions
$u_1,\ldots,u_k$ such that the trajectory defined by
$e_i=f(e_{i-1},u_i)$ for $i=1,\ldots,k$ satisfies $e_k\in G$. Assume the
terminal score is binary: it is $1$ when the computation reaches $G$ and $0$
otherwise. Suppose every environment-updating action used by the SOTM is
induced by at least one oracle call. Then, for every $\theta$,
\[
  \kappa_T(\theta;\alpha,\beta)
  ~\geq~
  \TOK_{\min}(\alpha,\beta)\, L\, \theta.
\]
\end{theorem}

\begin{proof}
Consider any SOTM in $\Fset(\theta)$, that is, any feasible SOTM achieving
quality at least $\theta$. Since $L$ is the goal depth, every successful
trajectory must contain at least $L$ environment-updating actions. By
assumption, each such action is induced by at least one oracle call. Hence, on a
successful trajectory, the computation must make at least $L$ oracle calls that
induce these actions.

Let $A$ be the event that the computation reaches $G$, and let $C$ be the total
token cost. If $\theta=0$, the claim is trivial. Suppose $\theta>0$. Since the
terminal score is the indicator of $A$, quality at least $\theta$ gives
$\Pr[A]\geq\theta>0$, so $\E_{\SOTM}[C\mid A]$ is well defined. Since each
oracle call has conditional expected token cost at least
$\TOK_{\min}(\alpha,\beta)$, the event $A$ implies
$\E_{\SOTM}[C\mid A]\geq \TOK_{\min}(\alpha,\beta)L$. Hence
\[
  \E_{\SOTM}[C]
  ~\geq~
  \E_{\SOTM}[C\mathbf 1_A]
  ~=~
  \Pr[A]\E_{\SOTM}[C\mid A]
  ~\geq~
  \TOK_{\min}(\alpha,\beta)L\,\Pr[A]
  \geq
  \TOK_{\min}(\alpha,\beta)L\,\theta.
\]
Taking the infimum over all SOTMs in $\Fset(\theta)$ gives the claimed lower
bound.
\end{proof}

The preceding lower bound does not require goal loss. It only uses the fact that
any successful trajectory must contain at least $L$ environment-updating
actions. It is also useful to isolate the opposite case, where no action can
cause goal loss. We call a task \emph{fully reversible} when goal loss has
probability zero at every call. Here ``fully reversible'' means reversible for
purposes of goal achievement: a call may be a detour that fails to advance the
computation toward the goal, but it does not make the goal unreachable.

\begin{theorem}[Goal-Depth Token Cost without Goal Loss]
\label{thm:reversible}
Suppose the task is fully reversible and the agentic oracle is the sole source
of progress: every environment-updating action used by the SOTM is induced by
an oracle response. Suppose each call independently yields a progress-enabling
response with probability $q \in (0,1]$, the goal has depth $L$, and the
terminal score is binary, equal to $1$ on reaching $G$ and $0$ otherwise. Assume
the number $N$ of oracle calls before halting has finite expectation.

Suppose the conditional expected token cost of each call lies between
$\TOK_{\min}(\alpha,\beta)$ and $\TOK_{\max}(\alpha,\beta)$. Then every SOTM in
$\Fset(\theta)$ satisfies $\E_{\SOTM}[N]\geq \theta L/q$. The token complexity
satisfies
\[
  \TOK_{\min}(\alpha,\beta)\, \frac{\theta L}{q}
  ~\leq~
  \kappa_T(\theta;\alpha,\beta)
  ~\leq~
  \TOK_{\max}(\alpha,\beta)\, \frac{L}{q}.
\]
If each call has the same conditional expected token cost
$\TOK_{\mathrm{call}}(\alpha,\beta)$, then
\[
  \lim_{\theta\uparrow1}\kappa_T(\theta;\alpha,\beta)
  =
  \frac{L}{q}\,\TOK_{\mathrm{call}}(\alpha,\beta).
\]
\end{theorem}

\begin{proof}
\emph{Upper bound.} Consider the SOTM that repeats oracle calls and applies
each progress-enabling response. Since the task is fully reversible, calls that
do not enable progress do not make the goal unreachable. The computation
reaches $G$ after $L$
progress-enabling responses. The number of calls $N$ needed to obtain
these $L$ responses is negative-binomial with $\E_{\SOTM}[N]=L/q$. Since each
call has conditional expected token cost at most $\TOK_{\max}(\alpha,\beta)$,
this gives the upper bound
\[
  \kappa_T(\theta;\alpha,\beta)
  \leq
  \TOK_{\max}(\alpha,\beta)\,\frac{L}{q}.
\]

\emph{Lower bound.} Consider any SOTM in $\Fset(\theta)$, that is, any feasible
SOTM achieving quality at least $\theta$, and let $N$ be its number of oracle
calls before halting. Let $Y_t$ be the indicator that the $t$-th call produces a
progress-enabling response. By the sole-source assumption, reaching $G$ requires
at least $L$ progress-enabling responses. Thus, on the event
$A=\{e_\sigma\in G\}$, we have
$\sum_{t=1}^{N}Y_t\geq L$.

Since the terminal score is the indicator of $A$, quality at least $\theta$
gives $\Pr[A]\geq\theta$. Hence
\[
  \theta
  \leq
  \Pr\!\left[\sum_{t=1}^{N}Y_t\geq L\right].
\]
By Wald's identity for the stopped sum of independent Bernoulli variables
\cite{Wald1947},
\[
  \E_{\SOTM}\!\left[\sum_{t=1}^{N}Y_t\right]
  =
  q\,\E_{\SOTM}[N].
\]
By Markov's inequality, for a nonnegative random variable
$Z$ and $a>0$, $\Pr[Z\geq a]\leq \E[Z]/a$. Applying this to
$Z=\sum_{t=1}^{N}Y_t$ and $a=L$ gives
\[
  \theta
  \leq
  \frac{q\,\E_{\SOTM}[N]}{L}.
\]
Thus $\E_{\SOTM}[N]\geq \theta L/q$. Applying the conditional lower bound call
by call gives expected token cost at least
$\TOK_{\min}(\alpha,\beta)\E_{\SOTM}[N]$, and hence at least
$\TOK_{\min}(\alpha,\beta)\theta L/q$. Taking the infimum over $\Fset(\theta)$
gives the lower bound.

If each call has the same conditional expected token cost
$\TOK_{\mathrm{call}}(\alpha,\beta)$, then the upper and lower bounds coincide
as $\theta\uparrow1$, giving the stated limit.
\end{proof}

\paragraph{Static environments.}
\label{sec:base}

In a static environment, the framework reduces to the stationary-oracle SOTM:
$E=\{e_0\}$ and $f(e_0,u)=e_0$ for every action $u$. In this case, there is no
separate agentic internal operation, so the agentic token-cost component is
zero.

\begin{proposition}[Base SOTM Recovery]
\label{thm:base}
Suppose the environment is static and the SOTM uses a stationary stochastic
oracle. Then there is no agentic internal operation.

Fix a full input context $\mathcal I$. Suppose the directing PTM can certify
whether a response is correct and stops at the first certified correct response.
If each oracle call independently succeeds with probability $p_{\mathcal I}>0$
and has expected token cost $\TOK_{\mathrm{call}}(\mathcal I;\alpha,\beta)$,
then, for the fixed input context $\mathcal I$, $\kag(1;\alpha,\beta)=0$,
$\kappa_T(1;\alpha,\beta)=\korch(1;\alpha,\beta)$, and the quality-$1$ token
complexity is
\[
  \kappa_T(1;\alpha,\beta)
  =
  \frac{\TOK_{\mathrm{call}}(\mathcal I;\alpha,\beta)}
       {p_{\mathcal I}}.
\]
\end{proposition}

\begin{proof}
Because the SOTM is a stationary-oracle SOTM, there is no agentic oracle
internal operation, so $\TOK_t^{\mathrm{ag}} \equiv 0$ for every turn $t$.
Thus $\Phi^{\mathrm{ag}} \equiv 0$, so
$\kag = 0$ and the per-call token-cost decomposition gives
$\Phi=\Phi^{\mathrm{orch}}$, hence
$\kappa_T(1;\alpha,\beta)=\korch(1;\alpha,\beta)$ for the fixed input context
$\mathcal I$. By assumption, the directing PTM certifies each response and stops
at the first certified correct one, so each call succeeds independently with
probability $p_{\mathcal I}>0$. If $N$ is the number of calls before halting,
then $N$ is geometric with
$\E[N]=1/p_{\mathcal I}$. Multiplying the expected number of calls by the
per-call expected token cost gives, for the fixed input context $\mathcal I$ and
quality $\theta=1$, the displayed formula for $\kappa_T(1;\alpha,\beta)$.
\end{proof}

\paragraph{Unavoidable risk.}
\label{subsec:forced-risk}

We call a reachable state $e$ an \emph{unavoidable risk point} with goal-loss
risk level $p_{\mathrm{loss}}>0$ if every action sequence from $e$ that reaches
$G$ has, as its next action, an action on which the agentic oracle causes goal
loss with probability at least $p_{\mathrm{loss}}$. In such a state, every route
to the goal requires exposure to this goal-loss risk level.

\begin{theorem}[Task-Level Goal-Loss Converse]
\label{thm:converse}
Suppose every action sequence from $e_0$ that reaches $G$ passes through at least
$m$ unavoidable risk points. At each such point, suppose the conditional
probability of goal loss, given that no goal loss has occurred so far, is at
least $p_{\mathrm{loss}}\in(0,1)$. Then every strategy, represented by $M$,
satisfies
\[
  \Pr_M[e_\sigma \in G] ~\leq~ (1 - p_{\mathrm{loss}})^{m},
\]
so the task is infeasible at quality $\theta > (1 - p_{\mathrm{loss}})^{m}$.
For a sequence of tasks with $m\to\infty$, the upper bound tends to zero,
independently of token budget.
\end{theorem}

\begin{proof}
Fix any strategy, represented by $M$. Any trajectory that reaches $G$ must pass
through at least $m$ unavoidable risk points $s_1,\dots,s_m$ in order. At each
such point, the goal is still reachable and the next action has goal-loss
probability at least $p_{\mathrm{loss}}$ given that no goal loss has occurred so
far. Reaching $G$ requires avoiding goal loss at all $m$ points, an event with
probability at most
$\prod_{i=1}^{m}(1 - p_{\mathrm{loss}}) = (1-p_{\mathrm{loss}})^m$. Hence
$\Pr_M[e_\sigma \in G] \leq (1-p_{\mathrm{loss}})^m$. Infeasibility and the limit
follow. This is the task-level converse to
Theorem~\ref{thm:goal-loss-avoidance}, item 3: that result bounds a strategy's
achievable quality by its accumulated risk exposure, while here the environment
imposes a minimum exposure under every strategy.
\end{proof}

\section{Conclusion, Future Directions, and Open Problems}
\label{sec:future}

This paper extends the SOTM framework with stationary stochastic oracles
\cite{Wang2026,Wang2026comp} to agentic SOTMs, whose oracles may have
autonomous goal-directed control and environment access. A central
distinction is that the PTM observes only the
query-response interface, while the agentic oracle may carry internal state,
perform internal operations, use tools or member models, and update an
environment. This creates two token-cost components: the orchestration token
cost visible to the PTM and the agentic token cost incurred inside the oracle.

The results show settings in which agentic SOTMs can have execution-time
token-cost advantages over SOTMs using stationary stochastic oracles on the same
task and at the same quality level. The advantage arises when state, intermediate
history, or environment-updating work need not be repeatedly encoded through
the query-response interface.
At the same time, autonomy introduces a new risk: actions may cause goal loss,
making the goal unachievable. Thus computing with agentic oracles involves both
execution-time token-cost advantages and goal-loss limitations.

Two broad future directions are especially important.

\paragraph{Multi-oracle SOTMs.}
This paper and the previous papers \cite{Wang2026,Wang2026comp} focus on
computing with a single oracle, either agentic or stationary stochastic. A
multi-oracle extension is especially important for vertical applications, where
a solution often coordinates several specialized systems rather than relying on
a single general-purpose oracle.

Formally, this would allow the PTM to coordinate several stationary stochastic
or agentic oracles. In such a model, the PTM would have access to oracles
$\Oracle_1,\ldots,\Oracle_m$, each through its own query-response tape. At each
oracle call step, the PTM could select one or more oracle indexes, send a query
or subquery to each selected oracle, and aggregate the corresponding responses.
The oracles would not communicate with one another; all routing, decomposition,
selection, and aggregation would be performed by the PTM.

Such a model could include stationary stochastic oracles, agentic oracles, or
hybrid collections of both, possibly supplied by different vendors or equipped
with different response distributions, tools, member models, or internal states.
Each oracle $\Oracle_j$ would come with its own query-dependent response
distributions and token-cost parameters $(\alpha_j,\beta_j)$, and possibly its
own tokenizer, reliability profile, tools, and internal state. This raises
several concrete problems: defining multi-oracle token complexity under heterogeneous
token-cost parameters; optimizing oracle-selection and query-routing policies;
extending the orchestration--agentic token-cost decomposition across several
oracles; analyzing goal-loss risk when several agentic oracles can update the
same environment; and comparing systems whose tokenizers or billing units are
not directly comparable.

\paragraph{Stochastic environments.}
Network resources, external services, concurrent processes, and other
interacting systems can make the environment stochastic: the same read or write
operation may lead to different observations or state updates on different
runs. Extending the deterministic environment model of this paper to such
stochastic environments is a natural next step. In such settings, goal-loss
probabilities and token costs may depend not only on the oracle's internal
randomness but also on random environment transitions. A key question is how to
separate oracle randomness from environment randomness.

Beyond these two broad extensions, several more specific problems remain open.

\paragraph{Cost--performance correlation in agentic operation.}
In the present model, expected agentic token cost is treated as a function of
the query, internal state, and environment state. In practice, additional
internal computation may raise the task score or improve the usefulness
of the final response to the PTM, correlating $\TOK_t^{\mathrm{ag}}$ with downstream task
performance. Characterizing this cost--performance tradeoff, and how the PTM
should exploit it when choosing queries, remains open.

\paragraph{Compositional agents and multi-level token complexity.}
The present model treats the internal mechanism of an agentic oracle
behaviorally. In more complex systems, an agentic oracle may delegate to
sub-agents, which may themselves delegate further. The two-component cost split
of Section~\ref{sec:decomp} should then extend to a multi-level recursion: the
total token cost of a depth-$d$ hierarchy should include orchestration token
costs at each level and agentic token costs at the leaves. A composition theorem
bounding root token complexity by the per-level orchestration costs and leaf
agentic token costs, and identifying when the recursion collapses to the
stationary-oracle SOTM framework, is a natural target.

\paragraph{Unknown goal-loss probabilities.}
Theorem~\ref{thm:converse} treats the goal-loss probability
$p_{\mathrm{loss}}$ and the count $m$ of unavoidable risk points as quantities
determined by the environment and oracle behavior. Estimating these quantities
from interaction, so that the PTM can decide whether to attempt or continue a
task under a given token budget and target quality, is an open problem.

\bibliographystyle{plainnat}
\bibliography{reference}

\end{document}